%% file: example_paper.tex
\documentclass{article}

\usepackage{microtype}
\usepackage{graphicx}
\usepackage{subcaption}
\usepackage{booktabs} 

\usepackage{hyperref}

\usepackage[preprint]{icml2026}

\usepackage{amsmath}
\usepackage{amssymb}
\usepackage{mathtools}
\usepackage{amsthm}
\usepackage{multirow}
\usepackage{dcolumn}
\usepackage{booktabs}
\usepackage{graphicx}
\usepackage{subcaption}
\newcolumntype{d}{D{.}{.}{2}}

\usepackage[capitalize,noabbrev]{cleveref}

\theoremstyle{plain}
\newtheorem{theorem}{Theorem}[section]

\theoremstyle{definition}

\theoremstyle{remark}

\icmltitlerunning{TORQ: Two-Level Orthogonal Rotation for MXFP4 Quantization}

\makeatletter
\gdef\icmlcorrespondingauthor@text{Dawei Yang}
\makeatother

\begin{document}
\twocolumn[
  \icmltitle{TORQ: Two-Level Orthogonal Rotation for MXFP4 Quantization}

  \icmlsetsymbol{equal}{*}
  \icmlsetsymbol{corresponding}{\textdagger}

  \begin{icmlauthorlist}
    \icmlauthor{Zukang Xu}{equal,houmo}
    \icmlauthor{Xing Hu}{equal,houmo}
    \icmlauthor{Dawei Yang}{houmo,corresponding}
  \end{icmlauthorlist}

  \icmlaffiliation{houmo}{Houmo AI}

  \icmlkeywords{Large Language Models, Quantization, MXFP4, Post-Training Quantization, Orthogonal Rotation}

  \vskip 0.3in
]

\printAffiliationsAndNotice{\icmlEqualContribution \textsuperscript{\textdagger}Corresponding author.}

\input{sec_0.paper}

\end{document}

%% file: sec_0.paper.tex

\input{sec_1.abstract}
\input{sec_2.introduction}
\input{sec_3.related_work}
\input{sec_4.preliminaries}
\input{sec_5.method}
\input{sec_6.experiments}
\input{sec_7.conclusion}
\input{sec_8.limitation}

\clearpage
\section{Impact Statement}
This paper presents work whose goal is to advance the field of Machine Learning. There are many potential societal consequences of our work, none which we feel must be specifically highlighted here.
\bibliography{example_paper}
\bibliographystyle{icml2026}

\newpage
\appendix
\onecolumn
\input{sec_9.Appendix}

%% file: sec_1.abstract.tex
\begin{abstract}

As Large Language Models (LLMs) advance toward practical deployment, the Microscaling FP4 (MXFP4) format has emerged as a cornerstone for next-generation low-bit inference, owing to its ability to balance high dynamic range with hardware efficiency. However, directly applying MXFP4 to LLM activation quantization inevitably leads to significant accuracy degradation. In this paper, we theoretically analyze the error structure of MXFP4 activation quantization, revealing that the root cause of this performance drop lies in two structural imbalances between activation distributions and the MXFP4 block floating-point format: (1).Extreme Inter-block Variance Imbalance: Under fixed-block microscaling, inter-block variance is highly uneven. Specifically, a minority of high-variance blocks dominate the total quantization error and force the shared scaling factor upward, resulting in the excessive coarsening of numerous small-magnitude activations within the block. 
(2).Intra-block Codebook Utilization Imbalance: The heavy-tailed distribution of real-world activations severely mismatches the MXFP4 geometric progression codebook, leading to extremely low effective codeword utilization and leaving significant representation capacity idle.
To address these challenges, we propose TORQ (Two-level Orthogonal Rotation for MXFP4 Quantization), a training-free Post-Training Quantization (PTQ) framework designed to reshape the geometric properties of the activation space through optimal coordinate transformations. 
At the macroscopic level, TORQ leverages the Schur-Horn theorem to redistribute activation energy via inter-block orthogonal rotation. Mathematically formalized as variance equalization, this process acts as a soft outlier suppression mechanism—preventing high-variance blocks from driving up the shared scaling factors and thereby preserving the precision of small-magnitude elements.
At the microscopic level, TORQ employs maximum-entropy-guided intra-block rotation to alleviate "codebook collapse." Actively dispersing the highly concentrated activation distribution, it effectively "activates" underutilized codewords to maximize the MXFP4 codebook’s information capacity.
Experiments on various mainstream LLMs such as LLaMA3 and Qwen3 show that TORQ significantly improves the accuracy of MXFP4 activation quantization compared to existing SOTA methods: on the Qwen3-32B model, the perplexity on WikiText is reduced to 8.43 (7.61 for BF16), and the average accuracy is increased from 38.40\% with direct RTN to 73.63\% (74.82\% for BF16), successfully bridging the gap between 4-bit floating-point quantization and full-precision inference.
\end{abstract}

%% file: sec_2.introduction.tex
\section{Introduction}

With the widespread adoption of Large Language Models (LLMs) across various domains~\cite{qwen3, llama3}, their massive parameter counts and the memory bandwidth demanded during inference have become major bottlenecks for practical deployment~\cite{outliers}. To run LLMs efficiently on resource-constrained edge devices or in large-scale data centers, Post-Training Quantization (PTQ) has emerged as the industry standard paradigm~\cite{quik,rsavq,rwkvquant,mquant,moequant}. While traditional INT4 quantization has achieved significant success in weight compression, its limited dynamic range struggles to accommodate the long-tail outliers prevalent in Transformer activations~\cite{outliers}. Consequently, the industry is shifting toward low-precision floating-point formats with higher dynamic ranges through hardware co-design. Notably, the MXFP4 (e2m1) standard introduced by the Open Compute Project (OCP), which employs Block-wise Shared Exponent technology, offers representation capabilities approaching FP8 while maintaining 4-bit compression rates. Supported natively by next-generation hardware such as NVIDIA Blackwell and AMD Ryzen AI~\cite{ostquant,singlequant}, MXFP4 is widely regarded as the cornerstone of next-generation low-bit inference.

\begin{figure}[htbp] 
    \centering  
    \includegraphics[width=0.45\textwidth]{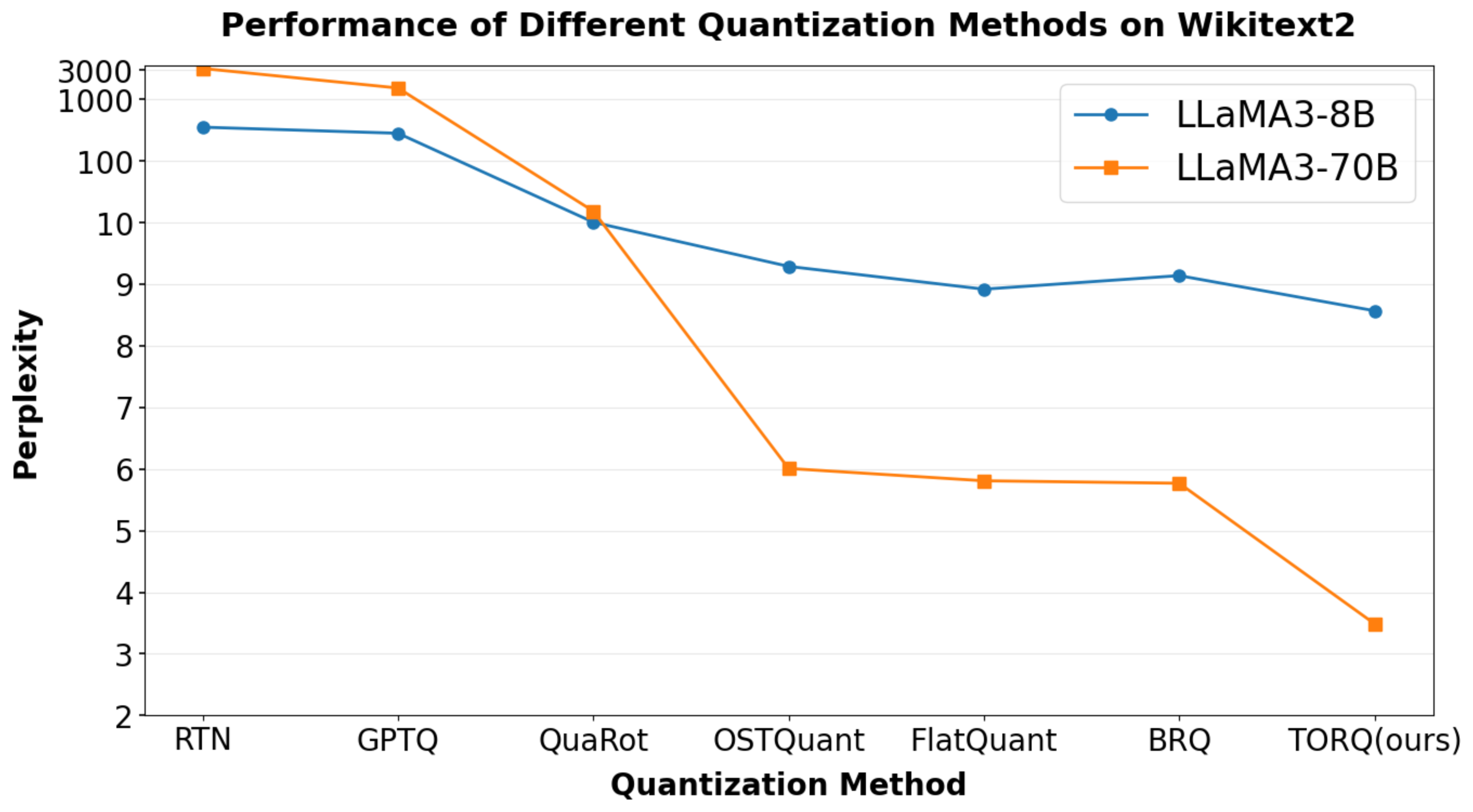}
    \caption{Performance of Different Quantization Methods on LLaMA3-8B and LLaMA3-70B.}
    \label{fig:quantization_performance}
\end{figure}
Despite the theoretical advantages of MXFP4, our empirical analysis reveals that directly applying it to LLM activation quantization encounters severe accuracy bottlenecks. State-of-the-art quantization methods, such as SingleQuant~\cite{singlequant}, DARTQuant~\cite{dartquant}, FlatQuant~\cite{flatquant}, and OSTQuant~\cite{ostquant}, have successfully addressed W4A4 integer quantization through rotation techniques. However, these methods are primarily designed to suppress outliers to fit the uniform grid of integers. This design overlooks the unique Block Floating Point (BFP) topological structure of MXFP4. As a result, simply transferring rotation strategies optimized for integers to MXFP4 often fails to exploit the potential of the floating-point format and may even be counterproductive.

\begin{figure*}[ht!] 
    \centering  
    \includegraphics[width=0.95\textwidth]{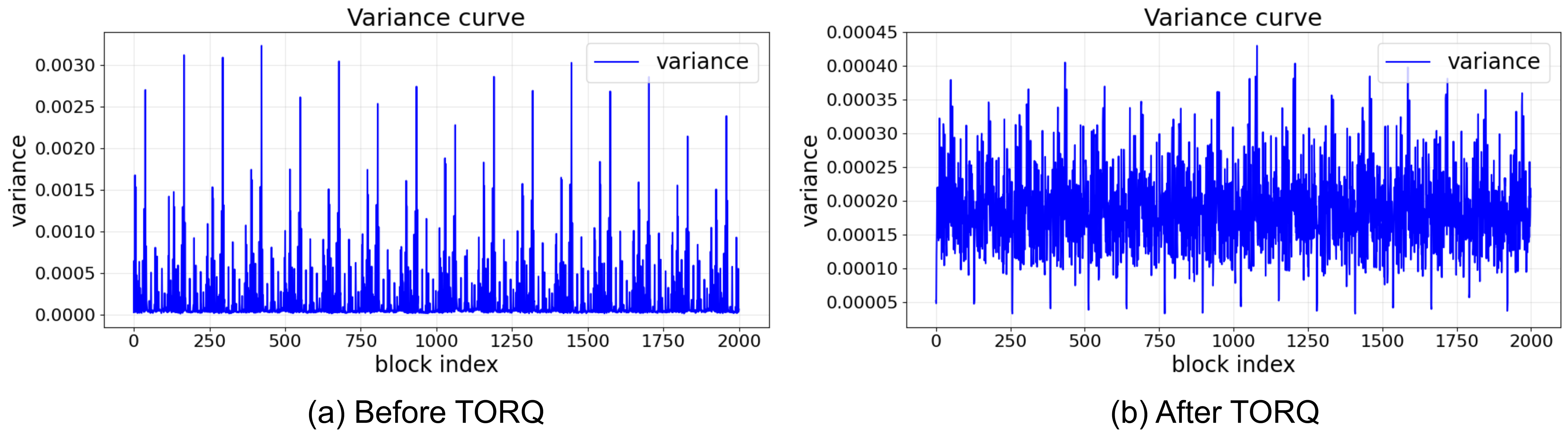}
    \caption{MXFP4 Quantization of each micro block's variance distribution: (a) Variance distribution of original data without TORQ, (b) Variance distribution after TORQ processing}
    \label{fig:motivation_1}
\end{figure*}
To address this challenge, this paper conducts an in-depth analysis of the error structure of MXFP4 activation quantization, revealing that the root cause of accuracy degradation lies in two structural conflicts between the activation distribution and the MXFP4 format: (1).Extreme Inter-block Variance Imbalance: The quantization accuracy of MXFP4 relies heavily on the block-wise shared scaling factor. When the activation vector is partitioned into fixed blocks (e.g., Block Size=32), as shown in Figure.\ref{fig:motivation_1}, block variance exhibits a significant heavy-tailed distribution. A minority of high-energy blocks not only dominate the total Mean Squared Error (MSE) but also force the shared exponent to rise sharply, causing over 90\% of small-magnitude elements within the same block to be zeroed out due to insufficient resolution. 
(2).Intra-block Codebook Utilization Imbalance (Codebook Collapse): The e2m1 codebook design of MXFP4 assumes data follows an ideal log-normal distribution. However, real-world activation distributions are often highly concentrated near zero, causing the vast majority of values to fall within the minimum quantization interval, while numerous codewords with strong representation capabilities for large values remain idle (i.e., Codebook Collapse). As shown in Figure.\ref{fig:motivation_2}, our statistics show that approximately 50\% of the codeword range is activated by only 10\% of the data in actual inference, resulting in severe bit waste.
\begin{figure*}[h] 
    \centering  
    \includegraphics[width=0.95\textwidth]{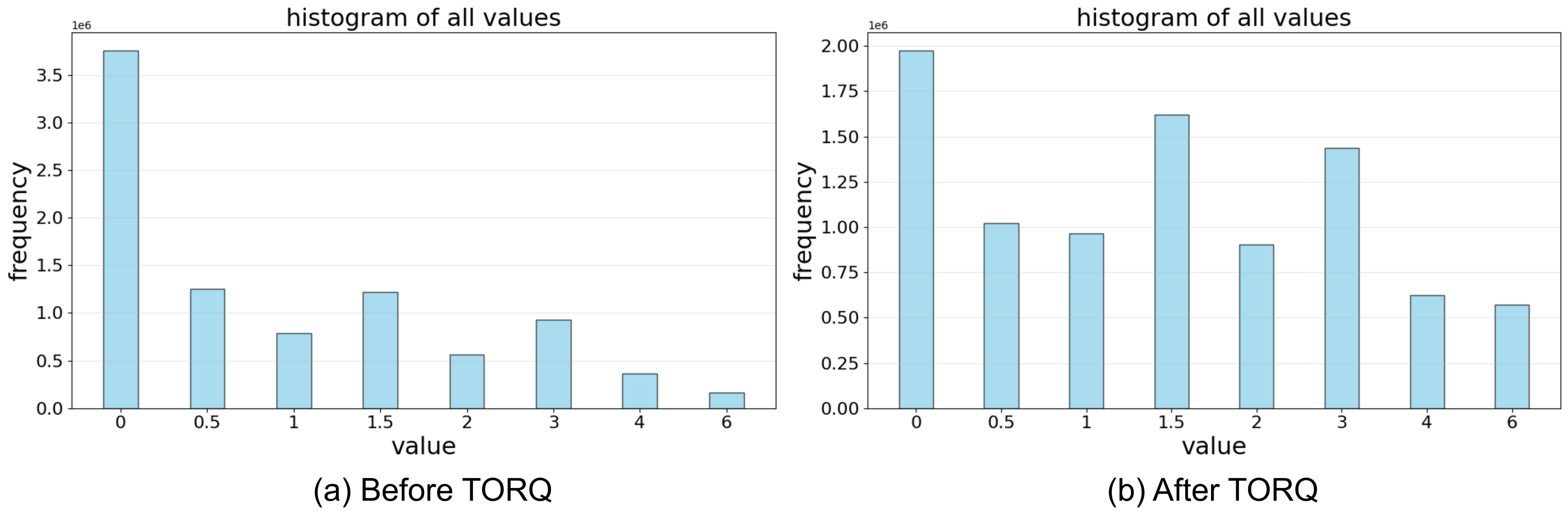}
    \caption{Histogram statistics of each codeword quantified by FP4(e2m1):(a) Codeword occupancy of the original data without TORQ, (b) Codeword occupancy after TORQ processing}
    \label{fig:motivation_2}
\end{figure*}

Based on these key insights, we propose TORQ (Two-level Orthogonal Rotation for MXFP4 Quantization). This is a retraining-free PTQ framework whose core philosophy is: Instead of passively adjusting quantization parameters to fit ill-conditioned activation distributions, we should actively rotate the activation coordinate system to align its distribution characteristics with the structural properties of MXFP4. 
TORQ decouples the optimization problem into two mathematically orthogonal sub-problems: 
(1).Macro-Equilibrium Rotation: To address inter-block variance imbalance, TORQ constructs an inter-block orthogonal transformation based on the Schur-Horn theorem to uniformly distribute activation energy across all blocks. We theoretically prove that this variance-equalized configuration is the optimal solution for minimizing global quantization error. (2).Micro-Alignment Rotation: To address low intra-block codebook utilization, TORQ introduces the principle of maximum entropy to execute orthogonal rotation within groups. This reshapes the local data distribution to "flow" uniformly into each codeword interval in the FP4 space, thereby maximizing effective information capacity.


TORQ enables the closed-form construction of rotation matrices using only a minimal volume of calibration data. Furthermore, the rotation matrices incorporated during the inference phase feature an extremely compact parameter footprint, imposing negligible additional computational overhead.Experimental evaluations on mainstream large language models (LLMs), including LLaMA and Qwen, demonstrate that TORQ achieves a significant performance superiority over prevailing state-of-the-art (SOTA) approaches. As illustrated in Figure.\ref{fig:quantization_performance}, this method consistently outperforms existing advanced quantization schemes across the entire Llama3 model series. Taking the Qwen3-8B model as a representative case: on the WikiText dataset, it attains a perplexity score of merely 10.26, representing a marginal increase of 1.26 compared with the full-precision floating-point baseline. Meanwhile, its average accuracy is boosted substantially from 30.73\% to 67.75\%, outperforming the FlatQuant method by a margin of 1.79 percentage points. Collectively, these results confirm that TORQ effectively mitigates the accuracy degradation challenge inherent to 4-bit floating-point quantization scenarios.
Our Contributions:
\begin{itemize}
    \item We systematically identify and formalize, for the first time, the two structural challenges hindering MXFP4 performance: Inter-block Variance Imbalance and Intra-block Codebook Utilization Imbalance.
    \item We propose TORQ, the first dual-level rotation framework explicitly designed for block floating-point structures. By integrating macro-variance equilibrium and micro-codebook alignment, it achieves a systematic minimization of quantization error.
    \item We establish a solid theoretical foundation for two-level rotation. We prove that under a convex error model, the variance equilibrium configuration is approximately optimal; meanwhile, we demonstrate that under the MXFP4 geometric structure, codeword occupancy equilibrium approximates the minimization of quantization error.
    \item We design a fine-tuning-free PTQ pipeline and verify the method's effectiveness, universality, and hardware friendliness across multiple LLMs and diverse downstream tasks.
\end{itemize}

%% file: sec_3.related_work.tex
\section{Related Work}
Traditional research on LLM quantization has primarily revolved around integer formats such as INT8 and INT4, focusing on reducing quantization accuracy loss through activation-aware strategies~\cite{illm,mambaquant}. For instance, GPTQ adjusts weight distributions via quantization-aware updates to mitigate the impact of outliers~\cite{gptq,gptaq}. AWQ proposes an activation-aware weight quantization strategy that protects salient weight channels based on activation statistics~\cite{awq}. SmoothQuant introduces scaling factors to redistribute the dynamic range between activations and weights, thereby alleviating overflow issues inherent to integer quantization~\cite{smoothquant}. However, these methods fundamentally combat the limitations of "fixed ranges" and fail to exploit the dynamic exponent characteristics of floating-point formats.

Recently, rotation techniques have emerged as a focal point in quantization research, representing the current state-of-the-art (SOTA) direction for W4A4 quantization. QuIP\#~\cite{quip_sharp} and QuaRot~\cite{quarot} introduce the Hadamard transform to smooth the distributions of weights and activations. While OSTQuant and FlatQuant have achieved SOTA results in W4A4 integer quantization, their objective functions aim to compress data dynamic range (Clip Ratio Optimization). This objective is incompatible with MXFP4, which requires fully utilizing dynamic range. Furthermore, although BRQ~\cite{brq} attempts rotation for MXFP4, it focuses solely on simple adjustments of intra-block distribution, neglecting the critical factor of inter-block energy imbalance.

The TORQ framework proposed in this paper is the first dual-level rotation framework explicitly designed for the structural characteristics of Block Floating Point (BFP). By combining "Macro-Equilibrium Rotation" and "Micro-Alignment Rotation," we systematically address the core challenges of MXFP4 activation quantization. This work addresses the lack of in joint inter-block and intra-block distribution optimization for block floating-point structures, establishing a new optimization paradigm for low-bit floating-point activation quantization.

%% file: sec_4.preliminaries.tex
\section{Preliminary}
\label{sec:preliminary}

\textbf{Notations.} Let $\mathbf{X} \in \mathbb{R}^{n \times d}$ denote the activation tensor of a specific layer, where $n$ is the number of tokens and $d$ is the feature dimension. We reshape $\mathbf{X}$ into $\mathbf{x} \in \mathbb{R}^{B \times K}$, representing $B$ blocks with $K$ elements each (typically $K=32$).

\textbf{MXFP4 Quantization Model.} The MXFP4 (e2m1) format employs a block-wise shared scaling factor. For a block vector $\mathbf{z}_b \in \mathbb{R}^K$, the quantization process is defined as:
\begin{equation}
\hat{\mathbf{z}}_b = s_b \cdot Q_{\text{mxFP4}}\left( \frac{\mathbf{z}_b}{s_b} \right)
\end{equation}
where $s_b = 2^{\lfloor \log_2 (\max |\mathbf{z}_b|) \rfloor}$ is the shared exponent, and $Q_{\text{mxFP4}}(\cdot)$ is the nearest-neighbor projection operator that maps normalized values to the e2m1 discrete codebook $\mathcal{C} = \{\pm 0, \pm 0.5, \pm 1, \pm 1.5, \pm 2, \pm 3, \pm 4, \pm 6\}$.

%% file: sec_5.method.tex
\section{Method}
\label{sec:method}
\begin{figure*}[h] 
    \centering  
    \includegraphics[width=0.95\textwidth]{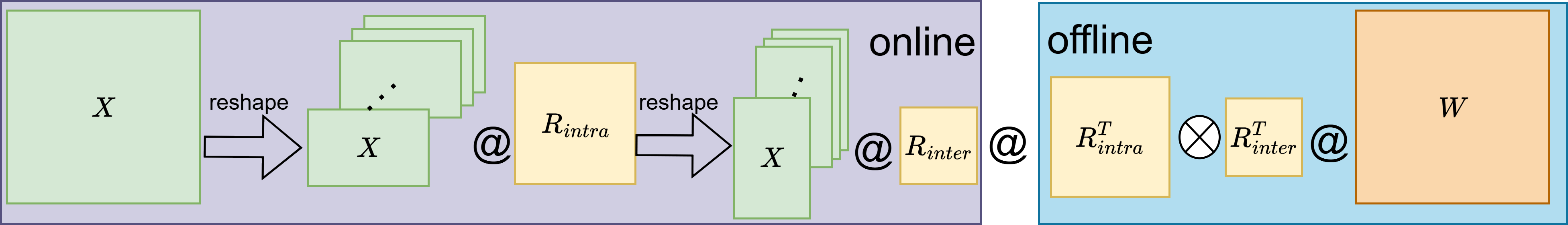}
    \caption{TORQ algorithm's complete process framework.}
    \label{fig:fram}
\end{figure*}

Our objective is to find an orthogonal transformation $\mathbf{R} \in \mathcal{O}(d)$ such that the quantization error of the rotated activations $\mathbf{Z} = \mathbf{X}\mathbf{R}$ in the MXFP4 format is minimized:
\begin{equation}
\min_{\mathbf{R} \in \mathcal{O}(d)} \mathbb{E} \left[ \| \mathbf{X}\mathbf{R} - Q_{\text{mxFP4}}(\mathbf{X}\mathbf{R}) \|_F^2 \right]
\end{equation}
Directly optimizing a high-dimensional orthogonal matrix $\mathbf{R}$ is computationally intractable and difficult to converge. To address this, TORQ decomposes $\mathbf{R}$ into two structured sub-transformations: Inter-block rotation $\mathbf{R}_{\text{inter}}$ and Intra-block rotation $\mathbf{R}_{\text{intra}}$, such that $\mathbf{R} = \mathbf{R}_{\text{inter}} \otimes \mathbf{R}_{\text{intra}}$. This decomposition not only reduces optimization complexity but, more importantly, achieves decoupled control over "energy" and "distribution." During the inference phase, the overhead of the inverse transformation can be nearly eliminated by fusing it with the weights of the adjacent linear layer. If the quantized activation $\hat{\mathbf{z}}$ is input to a linear layer $\mathbf{W} \in \mathbb{R}^{d' \times d}$ (where $d = B K$), the fused weight becomes $\mathbf{W}' = \mathbf{W} (\mathbf{R}_{\text{intra}}^\top \otimes \mathbf{R}_{\text{inter}}^\top)$, where $\otimes$ denotes the Kronecker product. Figure.\ref{fig:fram} illustrates the overall framework.

\subsection{Macro-Equilibrium Rotation: Inter-block Rotation based on the Schur-Horn Theorem}

The core objective of inter-block rotation is to achieve global equilibrium of variances across all blocks at the same position via orthogonal transformation, thereby eliminating the dominance of a few high-variance blocks over the total error. The design is based on rigorous convex optimization theory and an efficient engineering construction method.

\subsubsection{Theoretical Foundation: Optimality under Convex Error Models}

For the reshaped activation matrix with structure $B \times K$, assuming the expected mean within blocks is 0, we define the block variance vector $\boldsymbol{\sigma}^2 \in \mathbb{R}^B$, where $\sigma_b^2 = \|\mathbf{z}_b\|_2^2$.

\begin{theorem}[Optimality of Variance Equalization]
\label{thm:optimality}
Assuming the quantization MSE $\phi_b(\sigma^2)$ of block $b$ is a convex function of its input variance $\sigma^2$, the total MSE satisfies:
\begin{equation}
\min_{\mathbf{R}^{(k)} \in \mathcal{O}(B)} \sum_{b=1}^B \phi_b\left([\mathbf{R}^{(k)}\boldsymbol{\Sigma}_k\mathbf{R}^{(k)\top}]_{bb}\right) = \sum_{b=1}^B \phi_b\left(\frac{\mathrm{tr}(\boldsymbol{\Sigma}_k)}{B}\right)
\end{equation}
Equality holds if and only if $\mathrm{diag}(\mathbf{R}^{(k)}\boldsymbol{\Sigma}_k\mathbf{R}^{(k)\top}) = \frac{\mathrm{tr}(\boldsymbol{\Sigma}_k)}{B}\mathbf{1}$ (i.e., all block variances are equal), where $\mathcal{O}(B)$ denotes the group of orthogonal matrices of order $B$, and $\mathbf{1}$ is the all-ones vector.
\end{theorem}

\textit{Proof Sketch.} The Schur-Horn theorem~\cite{schur_horn} states that there exists an orthogonal matrix $\mathbf{R}^{(k)}$ such that the diagonal elements of the transformed covariance matrix can be any vector majorized by the eigenvalues. Under convex functions and constraints, the uniform allocation (constant vector) minimizes the objective function. The complete proof is provided in Appendix A.1.
Beyond the theoretical bounds established by convexity, the practical utility of variance equalization lies in suppressing local maxima. By redistributing energy from high-variance to low-variance blocks, TORQ acts as a dynamic 'soft clipper' that reduces the peak magnitudes determining the shared exponents ($s_b$)2. This effectively lowers the scaling factor ceiling, mitigating the severe resolution loss caused by outliers.



Proposition 4.2 (Empirical Error Behavior). While the strict convexity of the block MSE holds under high-rate quantization assumptions \cite{zamir}, in the coarse 4-bit regime, the error is dominated by the scaling factor $s_b$.

Analysis: The block-wise shared exponent $s
_b$ is determined by the maximum magnitude in the block. Outliers inflate $s_b$, increasing the quantization step size $\sigma=s_b \dot \sigma_{base}$
for all elements. By reducing the variance disparity between blocks (Theorem 4.1), TORQ effectively lowers the expectation of the maximum values, thereby minimizing the average scaling factor and reducing the global quantization noise. Although theoretical convexity is an approximation for discrete FP4, our extensive experiments confirm that variance equalization acts as an effective proxy for MSE minimization.

\subsubsection{Efficient Construction of Orthogonal Matrices: Schur-Horn Driven Givens Rotation}

Based on Theorem \ref{thm:optimality}, we need to construct an orthogonal matrix $\mathbf{R}_{\text{inter}}$ for all position $k$ such that $\mathrm{diag}(\mathbf{R}_{\text{inter}}\boldsymbol{\Sigma}_k\mathbf{R}^\top_{\text{inter}}) = c\mathbf{1}$, where $c=\mathrm{tr}(\boldsymbol{\Sigma}_k)/B$. We design an efficient iterative Givens rotation algorithm (Algorithm \ref{alg:givens}) that ensures convergence within $\mathcal{O}(B^2)$ time.

\begin{algorithm}[tb]
   \caption{Givens Rotation for Variance Equalization}
   \label{alg:givens}
\begin{algorithmic}[1]
   \STATE {\bfseries Input:} Covariance matrix $\boldsymbol{\Sigma} \in \mathbb{R}^{B \times B}$, convergence threshold $\epsilon$
   \STATE {\bfseries Output:} Orthogonal matrix $\mathbf{R} \in \mathcal{O}(B)$ such that $\mathrm{diag}(\mathbf{R}\boldsymbol{\Sigma}\mathbf{R}^\top) \approx c\mathbf{1}$
   \STATE Initialize $\mathbf{R} = \mathbf{I}_B$, calculate target $c = \mathrm{tr}(\boldsymbol{\Sigma})/B$
   \WHILE{$\max_i |\sigma_{ii} - c| > \epsilon$}
       \STATE Select block pair $(i, j)$ with opposite variance deviation signs, i.e., $(\sigma_{ii} - c)(\sigma_{jj} - c) < 0$
       \STATE Calculate rotation angle $\theta = \frac{1}{2}\arctan \frac{2\sigma_{ij}}{\sigma_{ii} - \sigma_{jj}}$
       \STATE Construct Givens rotation matrix $\mathbf{G}_{ij}(\theta)$
       \STATE Update $\boldsymbol{\Sigma} \leftarrow \mathbf{G}^\top \boldsymbol{\Sigma} \mathbf{G}$, $\mathbf{R} \leftarrow \mathbf{R}\mathbf{G}$
   \ENDWHILE
   \STATE {\bfseries return} $\mathbf{R}$
\end{algorithmic}
\end{algorithm}

\subsection{Micro-Alignment Rotation: Intra-block Rotation based on Codebook Alignment}

The goal of intra-block rotation is to reshape the intra-block activation distribution via orthogonal transformation in the normalized MXFP4 codeword space, ensuring that activation values uniformly occupy all available codewords, thereby maximizing the effective utilization of the 4-bit floating-point codebook.

\subsubsection{Computational Graph and Normalized FP4 Space}

Consider the activation matrix $\mathbf{X} \in \mathbb{R}^{B \times K}$ after inter-block rotation. We insert an identity transformation link into the computational graph:
\begin{equation}
\mathbf{X} \xrightarrow[]{\mathbf{S}^{-1}} \mathbf{X}_{\text{norm}} \xrightarrow[]{\mathbf{R}_{\text{intra}}} \mathbf{Z} \xrightarrow[]{Q_{\text{mxFP4}}} \hat{\mathbf{Z}} \xrightarrow[]{\mathbf{R}_{\text{intra}}^\top \mathbf{S}} \hat{\mathbf{X}}
\end{equation}
where $\mathbf{S} = \mathrm{diag}(s_1, \dots, s_B) \in \mathbb{R}^{B \times B}$ is the block-wise scaling matrix with shared scale $s_b$ for each block $b$, and $\mathbf{R}_{\text{intra}} \in \mathbb{R}^{K \times K}$ is the intra-block orthogonal rotation matrix. Under full precision, this link is strictly equivalent to an identity mapping. $\mathbf{S}^{-1}$ normalizes each block to the natural scale of MXFP4 (Normalized FP4 Space), and $\mathbf{R}_{\text{intra}}$ rearranges the statistical distribution of columns within this space.

\subsubsection{Codebook Occupancy Equilibrium Loss}

Let the set of positive codewords for MXFP4 (e2m1) be $\{c_1, \dots, c_J\}$ (where $J=8$), corresponding to decision boundaries $\{d_0=0, d_1, \dots, d_J=+\infty\}$. We define the \textbf{Codebook Occupancy Equilibrium Loss} as:
\begin{equation}
\mathcal{L}_{\text{code}}(\mathbf{S}, \mathbf{R}_{\text{intra}}) = \sum_{j=1}^{J} \left( \hat{p}_j(\mathbf{S}, \mathbf{R}_{\text{intra}}) - \frac{1}{J} \right)^2
\end{equation}
where $\hat{p}_j$ is the empirical codeword occupancy probability. This loss directly quantifies the "imbalance" of MXFP4 codeword usage, aligning precisely with the logarithmic geometric structure of MXFP4.

\subsubsection{Alternating Optimization Framework: S-step and R-step}

We employ an alternating coordinate descent strategy to iteratively optimize $\mathbf{S}$ and $\mathbf{R}_{\text{intra}}$.

\textbf{S-step (Update Block Scaling):} Fix $\mathbf{R}_{\text{intra}}$, and for each block $b$, map its maximum activation value to the maximum MXFP4 codeword $c_{\max}$:
\begin{equation}
s_b \leftarrow \text{clip\_to\_pow2}\left( \frac{\max_i |z_{b,i}|}{c_{\max}} \right)
\end{equation}

\textbf{R-step (Update Intra-block Rotation):} Fix $\mathbf{S}$ and update $\mathbf{R}_{\text{intra}}$ through a series of Givens rotations to monotonically decrease $\mathcal{L}_{\text{code}}$.
\begin{itemize}
    \item \textbf{Column Pair Selection:} Prioritize column pairs with high codebook imbalance and statistical complementarity based on the strategy detailed in Appendix A.5.
    \item \textbf{Angle Search:} For each selected pair $(p, q)$, solve for the optimal Givens rotation angle $\theta^\star$ via one-dimensional search: $\theta^\star = \arg\min_{\theta} \mathcal{L}_{\text{code}}(\mathbf{S}, \mathbf{R}_{\text{intra}} \mathbf{G}_{(p,q)}(\theta))$. Since the loss function is piecewise constant, the optimal solution can be found via finite enumeration (see Appendix A.4).
\end{itemize}

\begin{algorithm}[tb]
   \caption{Intra-block Rotation Construction (Micro Codebook Equilibrium)}
   \label{alg:intra}
\begin{algorithmic}[1]
   \STATE {\bfseries Initialize:} $\mathbf{R}_{\text{intra}}^{(0)} = \mathbf{I}_K$, $\mathbf{S}^{(0)} = \mathbf{I}_B$, $iter=0$
   \WHILE{$iter < MaxIter$ and $\Delta \mathcal{L}_{\text{code}} > \epsilon_{\text{intra}}$}
       \STATE \textbf{S-step:} Update scaling matrix $\mathbf{S}^{(iter+1)}$ using Eq. (5) based on current $\mathbf{R}_{\text{intra}}^{(iter)}$.
       \STATE \textbf{R-step:} Initialize $\mathbf{R}_{\text{current}} = \mathbf{R}_{\text{intra}}^{(iter)}$.
       \STATE Calculate imbalance $h_k$ and select top $K_{\text{top}}$ candidate columns.
       \STATE Select $P$ complementary pairs $\{(p_\ell, q_\ell)\}_{\ell=1}^P$ based on scores $H_{k,l}$.
       \FOR{$\ell=1$ to $P$}
           \STATE Compute optimal angle $\theta_\ell^\star$ for pair $(p_\ell, q_\ell)$ using exact solver.
           \STATE Update $\mathbf{R}_{\text{current}} \leftarrow \mathbf{R}_{\text{current}} \mathbf{G}_{(p_\ell, q_\ell)}(\theta_\ell^\star)$.
       \ENDFOR
       \STATE Update $\mathbf{R}_{\text{intra}}^{(iter+1)} = \mathbf{R}_{\text{current}}$.
       \STATE $iter \leftarrow iter + 1$.
   \ENDWHILE
   \STATE {\bfseries Output:} $\mathbf{R}_{\text{intra}}$, $\mathbf{S}$
\end{algorithmic}
\end{algorithm}

%% file: sec_6.experiments.tex
\section{Experiments}
\label{sec:experiments}

\subsection{Experiment Setup}

\textbf{Models and Datasets.} To comprehensively evaluate the universality of our algorithm, we selected the LLaMA-3~\cite{llama3} series and Qwen3~\cite{qwen3} series as foundation models, covering parameter scales from 8B to 70B and including both Dense and Mixture-of-Experts (MoE) architectures. Beyond standard perplexity evaluation on Wikitext2~\cite{wikitext2}, we assessed TORQ on a suite of zero-shot tasks, including ARC-Challenge and ARC-Easy~\cite{arc}, HellaSwag~\cite{hellaswag}, LAMBADA~\cite{lambada}, PIQA~\cite{piqa}, and WinoGrande~\cite{winogrande}. Furthermore, we evaluated performance on more challenging reasoning benchmarks, such as the multi-disciplinary knowledge task MMLU~\cite{mmlu}, the mathematical reasoning benchmark GSM8K~\cite{gsm8k}, and the code generation benchmark HumanEval~\cite{humaneval}.

\textbf{Baseline Methods.} We compared TORQ with current state-of-the-art low-bit quantization methods, including RTN, GPTQ~\cite{gptq}, QuaRot~\cite{quarot}, OSTQuant~\cite{ostquant}, FlatQuant~\cite{flatquant}, and BRQ~\cite{brq}. For fair comparison, all baseline methods were evaluated without GPTQ post-processing, examining solely the quantization performance of the methods themselves.

\textbf{Implementation Details.} All experiments were conducted on NVIDIA RTX 5090 GPUs. We applied MXFP4 quantization to both activations and weights. For calibration, we sampled 128 text segments from the Wikitext2 dataset. As TORQ is an efficient Post-Training Quantization (PTQ) framework, no fine-tuning was performed. For dense models (LLaMA3 and Qwen3), we evaluated MMLU, GSM8K, and HumanEval alongside other zero-shot tasks using the open-source \texttt{lm-evaluation-harness}~\cite{lm_eval}.

\subsection{Main Results}

Table \ref{tab:main_results} presents the main experimental results on the LLaMA3 and Qwen3 model families. \textbf{TORQ significantly outperforms SOTA methods across all tested models and tasks.} 
For example, on LLaMA3-8B, RTN and GPTQ nearly failed in the MXFP4 format (PPL > 280), confirming the challenge of directly applying MXFP4. In contrast, TORQ reduced the WikiText2 perplexity to \textbf{8.57}, surpassing the runner-up methods FlatQuant (8.92) and OSTQuant (9.29). In terms of average accuracy, TORQ achieved \textbf{69.32\%} on LLaMA3-8B, improving by over 37 percentage points compared to the baseline (RTN) and outperforming FlatQuant (68.34\%).

\input{tab_tab_main_results}

\subsection{Ablation Studies}

\textbf{Component Analysis.} To verify the contribution of individual components, we conducted ablation studies on Qwen3-8B (Table \ref{tab:ablation_component}). 
\begin{itemize}
    \item Using only Inter-block Rotation ($\mathbf{R}_{\text{inter}}$) reduced PPL from 3750.93 (RTN) to 24.73, indicating that resolving inter-block variance imbalance is fundamental for recovering model functionality.
    \item Using only Intra-block Rotation ($\mathbf{R}_{\text{intra}}$) lowered PPL to 20.36, demonstrating significant gains from optimizing codebook utilization.
    \item The complete TORQ ($\mathbf{R}_{\text{inter}} + \mathbf{R}_{\text{intra}}$) further reduced PPL to 10.26, boosting GSM8K accuracy to 85.95\%. This proves that macro-variance equilibrium and micro-geometric alignment are orthogonal and complementary.
\end{itemize}

\input{tab_tab_ablation_component}

\textbf{Calibration Dataset Robustness.} We evaluated the sensitivity of TORQ to different calibration datasets on Qwen3-14B (Table \ref{tab:ablation_data}). Results show minimal performance fluctuation (PPL variance < 0.8) across WikiText2, C4, and RedPajama, indicating that TORQ is robust and can capture statistical features from small amounts of generic text.

\input{tab_tab_ablation_data}

\textbf{Performance on Hard Tasks.} We assessed "deep intelligence" on MMLU, GSM8K, and HumanEval (Table \ref{tab:hard_tasks}). Traditional methods suffered severe degradation; for instance, GPTQ on Qwen3-8B achieved only 67.37\% on GSM8K and 30.84\% on HumanEval. TORQ recovered GSM8K accuracy to \textbf{86.82\%} (vs. 88.25\% FP16) and HumanEval to \textbf{61.15\%} (vs. 64.63\% FP16), thanks to the micro-alignment rotation preserving fine-grained features critical for reasoning.

\input{tab_tab_hard_tasks}
\vspace{-3mm}

\textbf{Generalizability to Other Formats.} We applied TORQ to MXINT4 and NVFP4 formats (Table \ref{tab:formats}). TORQ yielded significant improvements regardless of integer (MXINT4) or floating-point (NVFP4) constraints. Notably, on NVFP4, TORQ improved GSM8K accuracy on Qwen3-8B from 42.67\% (RTN) to 87.08\%, confirming the format-agnostic nature of our geometric optimization.

\input{tab_tab_formats}

\textbf{Adaptability to MoE Architectures.} Experiments on Qwen3-30B-A3B (MoE) demonstrated TORQ's effectiveness (Table \ref{tab:moe}). On GSM8K, TORQ achieved 83.83\%, far surpassing GPTQ (37.89\%) and QuaRot (74.18\%), proving its ability to mitigate the heterogeneity of activation distributions across experts.

\input{tab_tab_moe}

\subsection{Efficiency Analysis}

\textbf{Speedup and Memory Saving.} As shown in Table \ref{tab:speedup}, on Qwen3-30B-A3B with a 2k context length, TORQ (W4A4) achieves a decoding speed of 220 Toks/s, representing a \textbf{3.5$\times$ speedup} over BF16 (63 Toks/s), while reducing memory usage from 56.9GB to 15.3GB.

\input{tab_tab_speedup}

\textbf{Quantization Time Comparison.} Unlike training-dependent methods, TORQ's analytical construction is extremely efficient. As detailed in Table \ref{tab:time_cost}, TORQ calibrates Qwen3-32B in just 1920 seconds, which is approximately \textbf{10$\times$ faster than FlatQuant} and \textbf{7$\times$ faster than OSTQuant}, highlighting its practical deployment value.

\input{tab_tab_time_cost}

%% file: tab_tab_main_results.tex
\begin{table*}[h]
\centering
\caption{Main results on LLaMA3 and Qwen3 families. "Wiki2" reports Perplexity ($\downarrow$), and "Avg Acc" reports average accuracy on zero-shot tasks ($\uparrow$).}
\label{tab:main_results}
\resizebox{\textwidth}{!}{
\begin{tabular}{lcccccccccc}
\toprule
 & \multicolumn{2}{c}{\textbf{LLaMA3-8B}} & \multicolumn{2}{c}{\textbf{LLaMA3-70B}} & \multicolumn{2}{c}{\textbf{Qwen3-8B}} & \multicolumn{2}{c}{\textbf{Qwen3-14B}} & \multicolumn{2}{c}{\textbf{Qwen3-32B}} \\
\cmidrule(lr){2-3} \cmidrule(lr){4-5} \cmidrule(lr){6-7} \cmidrule(lr){8-9} \cmidrule(lr){10-11}
\textbf{Method} & Wiki2 & Avg Acc & Wiki2 & Avg Acc & Wiki2 & Avg Acc & Wiki2 & Avg Acc & Wiki2 & Avg Acc \\
\midrule
Baseline (FP16) & 6.14 & 73.18 & 2.87 & 79.95 & 9.00 & 70.43 & 8.64 & 73.81 & 7.61 & 74.82 \\
\midrule
RTN & 354.37 & 32.03 & 3180.43 & 30.92 & 3750.93 & 30.73 & 2219.90 & 35.22 & 274.93 & 38.40 \\
GPTQ & 282.29 & 33.60 & 1532.08 & 34.82 & 1633.38 & 35.54 & 368.24 & 39.91 & 59.31 & 57.62 \\
QuaRot & 10.16 & 62.13 & 15.38 & 62.87 & 23.50 & 55.35 & 20.11 & 60.21 & 18.86 & 61.23 \\
OSTQuant & 9.29 & 65.92 & 6.01 & 73.93 & 19.75 & 54.40 & 16.85 & 64.06 & 14.62 & 67.19 \\
FlatQuant & 8.92 & 68.34 & OOM & OOM & 11.21 & 66.96 & 13.32 & 69.87 & OOM & OOM \\
BRQ & 9.14 & 67.76 & - & - & 13.70 & 64.73 & - & - & - & - \\
\textbf{TORQ (ours)} & \textbf{8.57} & \textbf{69.32} & \textbf{3.49} & \textbf{76.28} & \textbf{10.26} & \textbf{67.75} & \textbf{12.87} & \textbf{70.51} & \textbf{8.43} & \textbf{73.63} \\
\bottomrule
\end{tabular}
}
\end{table*}

%% file: tab_tab_ablation_component.tex
\begin{table*}[h]
\centering
\caption{Ablation study of rotation components on Qwen3-8B.}
\label{tab:ablation_component}
{
\begin{tabular}{cc|cccccc}
\toprule
$\mathbf{R}_{\text{inter}}$ & $\mathbf{R}_{\text{intra}}$ & \textbf{Wiki2} & \textbf{ARC-C} & \textbf{ARC-E} & \textbf{Hella} & \textbf{PIQA} & \textbf{GSM8K} \\
\midrule
\multicolumn{2}{c|}{Baseline (FP16)} & 9.00 & 56.74 & 80.93 & 74.98 & 77.37 & 88.25 \\
\midrule
$\times$ & $\times$ & 3750.93 & 22.94 & 25.07 & 28.73 & 52.11 & 48.93 \\
$\checkmark$ & $\times$ & 24.73 & 40.33 & 60.18 & 55.30 & 61.46 & 65.98 \\
$\times$ & $\checkmark$ & 20.36 & 48.41 & 69.17 & 61.01 & 68.83 & 70.84 \\
$\checkmark$ & $\checkmark$ & \textbf{10.26} & \textbf{53.65} & \textbf{79.36} & \textbf{72.06} & \textbf{74.11} & \textbf{85.95} \\
\bottomrule
\end{tabular}
}
\end{table*}

%% file: tab_tab_ablation_data.tex
\begin{table}[h]
\centering
\caption{Robustness across different calibration datasets on Qwen3-14B.}
\label{tab:ablation_data}
\resizebox{0.5\textwidth}{!}{
\begin{tabular}{lcccc}
\toprule
\textbf{Calibration Data} & \textbf{Wiki2} & \textbf{ARC-C} & \textbf{ARC-E} & \textbf{Hella} \\
\midrule
Wiki2 & 12.87 & 55.27 & 81.74 & 76.83 \\
C4 & 13.63 & 56.11 & 81.09 & 78.02 \\
Red\_Pajama & 13.11 & 55.78 & 80.93 & 77.37 \\
\bottomrule
\end{tabular}
}
\end{table}

%% file: tab_tab_hard_tasks.tex
\begin{table}[h]
\centering
\caption{Performance on complex reasoning tasks (Qwen3-8B).}
\label{tab:hard_tasks}
\resizebox{0.5\textwidth}{!}{
\begin{tabular}{lcccc}
\toprule
\textbf{Method} & \textbf{MMLU} & \textbf{GSM8K} & \textbf{HumanEval} & \textbf{MathQA} \\
\midrule
Baseline & 74.71 & 88.25 & 64.63 & 49.65 \\
GPTQ & 50.29 & 67.37 & 30.84 & 36.92 \\
QuaRot & 68.88 & 79.30 & 53.37 & 41.12 \\
\textbf{TORQ (ours)} & \textbf{70.82} & \textbf{86.82} & \textbf{61.15} & \textbf{47.94} \\
\bottomrule
\end{tabular}
}
\end{table}

%% file: tab_tab_formats.tex
\begin{table}[!htbp]  
\centering
\renewcommand\arraystretch{1.15}  
\caption{Applicability to different microscaling formats (Qwen3-8B).}
\label{tab:formats}  
\resizebox{0.5\textwidth}{!}{
\begin{tabular}{c c c cccc}  
\toprule
\textbf{Format} & \textbf{Method} & \textbf{ARC-C} & \textbf{GSM8K} & \textbf{Hella} & \textbf{PIQA} & \textbf{AVG} \\
\midrule
\multirow{2}{*}{MXFP4} & RTN & 22.94 & 35.87 & 28.73 & 52.11 & 34.92 \\
                       & \textbf{TORQ} & \textbf{53.65} & \textbf{86.82} & \textbf{72.06} & \textbf{74.11} & \textbf{71.66} \\
\midrule
\multirow{2}{*}{MXINT4} & RTN & 23.83 & 37.90 & 30.15 & 50.83 & 35.68 \\
                        & \textbf{TORQ} & \textbf{51.98} & \textbf{85.12} & \textbf{72.72} & \textbf{73.51} & \textbf{70.83} \\
\midrule
\multirow{2}{*}{NVFP4} & RTN & 30.17 & 42.67 & 37.33 & 60.12 & 42.57 \\
                       & \textbf{TORQ} & \textbf{54.38} & \textbf{87.08} & \textbf{73.10} & \textbf{75.83} & \textbf{72.60} \\
\bottomrule
\end{tabular}
}
\end{table}

%% file: tab_tab_moe.tex
\begin{table*}[h]
\centering
\caption{Performance on MoE Architecture (Qwen3-30B-A3B).}
\label{tab:moe}
\begin{tabular}{lcccccc}
\toprule
\textbf{Method} & \textbf{ARC-E} & \textbf{ARC-C} & \textbf{Hella} & \textbf{GSM8K} & \textbf{WinG} & \textbf{PIQA} \\
\midrule
Baseline & 79.25 & 56.40 & 79.60 & 85.44 & 70.32 & 80.30 \\
RTN & 24.58 & 26.80 & 25.52 & 22.45 & 52.87 & 51.04 \\
GPTQ & 31.15 & 34.29 & 40.19 & 37.89 & 58.90 & 62.81 \\
QuaRot & 70.45 & 53.19 & 71.14 & 74.18 & 62.51 & 71.11 \\
OSTQuant & 75.84 & 54.03 & 77.23 & 80.28 & 65.80 & 75.09 \\
\textbf{TORQ (ours)} & \textbf{77.10} & \textbf{55.81} & \textbf{78.12} & \textbf{83.83} & \textbf{68.36} & \textbf{77.93} \\
\bottomrule
\end{tabular}
\end{table*}

%% file: tab_tab_speedup.tex
\begin{table}[h]
\centering
\caption{Inference speedup and memory footprint on Qwen3-30B-A3B.}
\label{tab:speedup}
\begin{tabular}{lcc}
\toprule
\textbf{Precision} & \textbf{Memory (GB)} & \textbf{Decoder Toks/s} \\
\midrule
BF16 & 56.9 & 63 \\
\textbf{TORQ (W4A4)} & \textbf{15.3} & \textbf{220} \\
\bottomrule
\end{tabular}
\end{table}

%% file: tab_tab_time_cost.tex
\begin{table}[h]
\centering
\caption{Comparison of quantization calibration time (seconds).}
\label{tab:time_cost}
\begin{tabular}{lccc}
\toprule
\textbf{Model} & \textbf{OSTQuant} & \textbf{FlatQuant} & \textbf{TORQ} \\
\midrule
Qwen3-8B & 2786s & 3247s & \textbf{427s} \\
Qwen3-14B & 5131s & 6489s & \textbf{893s} \\
Qwen3-32B & 14834s & 18619s & \textbf{1920s} \\
\bottomrule
\end{tabular}
\end{table}

%% file: sec_7.conclusion.tex
\section{Conclusion}
This paper addresses the core challenges faced by MXFP4, as a next-generation low-precision floating-point format, in the activation quantization of Large Language Models (LLMs) and proposes the TORQ two-level orthogonal rotation preprocessing framework. Through systematic analysis, it reveals the structural issues of uneven cross-block variance and unequal usage of intra-block FP4 codewords. We innovatively elevate quantization optimization from the level of local parameter adjustment to that of coordinate system design: Inter-block rotation, based on the Schur-Horn theorem and convex optimization theory, achieves the global balanced distribution of cross-block variance, eliminating the dominance of a few high-variance blocks over the total error; Intra-block rotation, through the collaborative design of $log2$ domain target guidance and linear domain orthogonal transformation, makes the distribution of activation logarithmic amplitudes approximate uniformity, maximizing the utilization of MXFP4 codewords. Theoretically, we establish strict proofs for the error lower bounds and optimization properties of the two-level rotations respectively; In practice, this method requires no fine-tuning, has low calibration costs, and is hardware-friendly. On a series of models such as LLaMA and Qwen, as well as various tasks including WikiText, MMLU, and GSM8K, it significantly narrows the precision gap between MXFP4 quantization and full-precision inference, with performance surpassing existing SOTA methods by 2–5\%. TORQ not only provides the first systematic solution for MXFP4 activation quantization but also constructs a new paradigm of "structure-aware orthogonal transformation optimization," laying a crucial foundation for the application of low-bit floating-point formats in the efficient deployment of LLMs.

%% file: sec_8.limitation.tex
\section{Limitation}
Although TORQ demonstrates an excellent trade-off between accuracy and efficiency, there are several limitations that warrant further exploration: First, the current intra-block rotation focuses on matching first-order statistics (the expectation of log magnitudes) and fails to fully utilize the higher-order statistical information of the activation distribution (such as kurtosis and skewness). In scenarios where some activation distributions are extremely asymmetric or multimodal, it may not fully unleash the expressive potential of FP4. Second, this paper uses a fixed block size (e.g., 32) for experiments, while the activation characteristics of different model layers and different tasks vary significantly. The dynamic selection mechanism for adaptive block sizes and its impact on quantization accuracy and hardware efficiency have not been thoroughly studied. Finally, although we verified the effectiveness of the method on MoE models, we did not design specific optimization strategies for the heterogeneity of activation distributions among Experts. For ultra-large-scale MoE models (e.g., with more than a thousand Experts), how to further improve cross-Expert quantization consistency remains to be explored. These limitations also point out directions for future work, such as the design of adaptive rotation combined with higher-order statistics and the joint optimization of block size and model characteristics.

%% file: sec_9.Appendix.tex
\section{Appendices}
\input{appendix_A.1}
\input{appendix_A.2}
\input{appendix_A.3}
\input{appendix_A.4}
\input{appendix_A.5}
\input{appendix_A.6}
\input{appendix_A.7}
\input{appendix_A.8}

%% file: appendix_A.1.tex
\subsection*{Appendix A.1: Proof of Theorem 1 (Optimality of Variance Equalization)}
\label{app:proof_thm1}

\begin{theorem}[Optimality of Variance Equalization]
Under the \textit{convex block error model} (i.e., the MSE $\phi_b(\sigma^2)$ of block $b$ is a convex function of $\sigma^2$), the total MSE satisfies:
\begin{equation}
\min_{\mathbf{R} \in \mathcal{O}(B)} \sum_{b=1}^B \phi_b \left( [\mathbf{R} \boldsymbol{\Sigma}_k \mathbf{R}^\top]_{bb} \right) = \sum_{b=1}^B \phi_b \left( \frac{\mathrm{tr}(\boldsymbol{\Sigma}_k)}{B} \right)
\end{equation}
Equality holds if and only if $\mathrm{diag}(\mathbf{R} \boldsymbol{\Sigma}_k \mathbf{R}^{\top}) = \frac{\mathrm{tr}(\boldsymbol{\Sigma}_k)}{B} \mathbf{1}$.
\end{theorem}

\begin{proof}
\textbf{Step 1: Problem Transformation.}
By the property of orthogonal transformations, the total variance is conserved:
\begin{equation}
\sum_{b=1}^B [\mathbf{R} \boldsymbol{\Sigma}_k \mathbf{R}^\top]_{bb} = \mathrm{tr}(\boldsymbol{\Sigma}_k) = \text{constant}
\end{equation}
Let $\mathbf{d} = \mathrm{diag}(\mathbf{R} \boldsymbol{\Sigma}_k \mathbf{R}^\top)$. The optimization problem can be reformulated as:
\begin{equation}
\min_{\mathbf{d} \in \mathcal{D}} \sum_{b=1}^B \phi_b(d_b) \quad \text{s.t.} \quad \sum_{b=1}^B d_b = \mathrm{tr}(\boldsymbol{\Sigma}_k)
\end{equation}
where $\mathcal{D}$ is the set of achievable diagonal vectors.

\textbf{Step 2: Application of Schur-Horn Theorem.}
The Schur-Horn theorem~\cite{schur_horn} states that $\mathcal{D}$ is exactly the set of vectors \textit{majorized} by the eigenvalues $\boldsymbol{\lambda} = \mathrm{eig}(\boldsymbol{\Sigma}_k)$. That is:
\begin{equation}
\mathcal{D} = \left\{ \mathbf{d} \in \mathbb{R}^B : \sum_{i=1}^j d_{[i]} \leq \sum_{i=1}^j \lambda_{[i]} \text{ for } j=1,\dots,B-1, \text{ and } \sum_{i=1}^B d_i = \sum_{i=1}^B \lambda_i \right\}
\end{equation}
where $d_{[i]}$ denotes the $i$-th largest component of $\mathbf{d}$.

\textbf{Step 3: Convex Optimization and Majorization Theory.}
According to Karamata's inequality for convex functions~\cite{karamata}: if $\phi$ is convex and $\mathbf{d} \prec \mathbf{e}$ ($\mathbf{d}$ is majorized by $\mathbf{e}$), then $\sum \phi(d_i) \le \sum \phi(e_i)$.
The constant vector $\mathbf{c} = \frac{\mathrm{tr}(\boldsymbol{\Sigma}_k)}{B} \mathbf{1}$ is the \textit{most uniform} vector in $\mathcal{D}$ (i.e., $\mathbf{c} \prec \mathbf{d}$ for any $\mathbf{d} \in \mathcal{D}$).
Since $\phi_b$ is convex, the inequality $\sum \phi_b(c) \le \sum \phi_b(d_b)$ holds for all $\mathbf{d} \in \mathcal{D}$.

\textbf{Step 4: Equality Condition.}
Equality holds if and only if:
1) $\mathbf{d} = \mathbf{c}$ (from the equality condition of Karamata's inequality).
2) $\mathbf{c} \in \mathcal{D}$ (from Schur-Horn, since $\mathbf{c}$ is majorized by $\boldsymbol{\lambda}$).
The condition that $\mathbf{c}$ is majorized by $\boldsymbol{\lambda}$ is equivalent to $\lambda_{\min} \le c \le \lambda_{\max}$, which is naturally satisfied as $\frac{\mathrm{tr}(\boldsymbol{\Sigma}_k)}{B}$ lies between the minimum and maximum eigenvalues.

\textbf{Step 5: Constructive Existence.}
By the Schur-Horn theorem, there exists an orthogonal matrix $\mathbf{R}$ such that $\mathrm{diag}(\mathbf{R} \boldsymbol{\Sigma}_k \mathbf{R}^\top) = \mathbf{c}$. This can be constructed via eigenvalue decomposition followed by a series of Givens rotations (as detailed in Algorithm 1) or via a closed-form Householder solution.
Specifically:
1) Perform eigen-decomposition $\boldsymbol{\Sigma}_k = \mathbf{U} \boldsymbol{\Lambda} \mathbf{U}^\top$.
2) Solve for $\mathbf{Q}$ such that $\mathrm{diag}(\mathbf{Q} \boldsymbol{\Lambda} \mathbf{Q}^\top) = \mathbf{c}$.
3) Let $\mathbf{R} = \mathbf{Q}\mathbf{U}^\top$.
\end{proof}

\textbf{Key Insight:} This theorem indicates that under reasonable convex error assumptions (as proven for MXFP4 in Appendix A.2), a rotation configuration that perfectly equalizes variances across all blocks at a given position (i.e., $d_b = \mathrm{tr}(\boldsymbol{\Sigma}_k)/B$ for all $b$) is the global optimal solution. This theoretically underpins the design goal of inter-block rotation: to eliminate the dominance of a few high-variance blocks and distribute the error uniformly.

%% file: appendix_A.2.tex
\subsection*{Appendix A.2: Convexity Analysis of the MXFP4 Quantization Error Model}
\label{app:convexity_analysis}

For 4-bit MXFP4 (e2m1) block floating-point quantization, the Mean Squared Error (MSE) of block $b$ is jointly determined by the \textbf{Distribution Shape} and \textbf{Dynamic Range Coverage}. According to asymptotic quantization theory \cite{zamir}, its lower bound is given by:
\begin{equation}
D_b \geq \underbrace{\frac{\Delta_{\min}^2}{12} \cdot K}_{\text{Granularity Error}} \cdot \underbrace{\exp\left( 2 \cdot \mathrm{Var}(\log_2 |\mathbf{a}_b|) \right)}_{\text{Shape Factor}} \cdot \underbrace{\exp\left( D_{\text{KL}}(f_U \| \mathcal{U}_{[L,U]}) \right)}_{\text{Matching Factor}}
\end{equation}
where:
\begin{itemize}
    \item $\Delta_{\min} = 2^{e_{\min}-1}$ is the minimum quantization step size.
    \item $\mathbf{a}_b = \mathbf{z}_b / s_b$ is the normalized activation vector.
    \item $f_U$ is the probability density function of the logarithmic magnitude $u = \log_2 |\mathbf{a}|$.
    \item $\mathcal{U}_{[L,U]}$ denotes the uniform distribution over the interval $[L, U]$, where $L=e_{\min}$ and $U=e_{\max} + \log_2 1.5$.
\end{itemize}

The total quantization error is defined as $D_{\text{total}} = \sum_{b=1}^B D_b$. Our objective is to minimize $D_{\text{total}}$ via orthogonal transformations:
\begin{equation}
\min_{\mathbf{R}_{\text{inter}}, \mathbf{R}_{\text{intra}}} \mathbb{E}\left[ \|\mathbf{x} - \hat{\mathbf{x}}\|_2^2 \right]
\end{equation}
Subject to the following constraints:
\begin{enumerate}
    \item Orthogonality: $\mathbf{R}_{\text{inter}} \in \mathcal{O}(B)$ and $\mathbf{R}_{\text{intra}} \in \mathcal{O}(K)$.
    \item Training-free: No model fine-tuning is required.
    \item Data-efficiency: Low calibration data requirement.
\end{enumerate}

%% file: appendix_A.3.tex
\subsection*{Appendix A.3: Inference Process}
\label{app:inference_process}

We present the online inference process of TORQ in Algorithm \ref{alg:inference}. The core advantage of TORQ is that the computational overhead of the inverse transformation can be absorbed into the linear layers, resulting in zero overhead during online inference.

\begin{algorithm}[h]
   \caption{TORQ Online Inference}
   \label{alg:inference}
\begin{algorithmic}[1]
   \STATE {\bfseries Input:}
   \STATE \quad Input activation vector $\mathbf{x} \in \mathbb{R}^d$ (where $d = B \cdot K$)
   \STATE \quad Offline parameters: $\mathbf{R}_{\text{inter}}$, $\mathbf{R}_{\text{intra}}$, $\mathbf{S}$
   \STATE \quad Weight matrix of the subsequent linear layer $\mathbf{W} \in \mathbb{R}^{d' \times d}$
   \STATE {\bfseries Output:} Quantized activation $\hat{\mathbf{x}}$ or output of the linear layer
   
   \STATE \textbf{1. Data Reshape:} Reshape $\mathbf{x}$ into matrix $\mathbf{X} \in \mathbb{R}^{B \times K}$.
   
   \STATE \textbf{2. Apply Inter-block Rotation:}
   
   \STATE Rotate: $\mathbf{Y} = \mathbf{R}_{inter} \mathbf{X}$
   
   \STATE Reassemble to obtain $\mathbf{Y} \in \mathbb{R}^{B \times K}$.
   
   \STATE \textbf{3. Apply Intra-block Rotation:}
   
   \STATE $\mathbf{Z} = \mathbf{Y} \mathbf{R}_{\text{intra}}$
   
   \STATE \textbf{4. MXFP4 Quantization:}
   \FOR{$b=1$ to $B$}
       \STATE Extract block: $\mathbf{z}_b = [z_{b,1}, \dots, z_{b,K}]^\top$
       \STATE Get shared scale: $s_b$ (from $\mathbf{S}$)
       \STATE Normalize: $\mathbf{a}_b = \mathbf{z}_b / s_b$
       \STATE Quantize: $\hat{\mathbf{a}}_b = Q_{\text{mxFP4}}(\mathbf{a}_b)$
       \STATE De-quantize: $\hat{\mathbf{z}}_b = s_b \cdot \hat{\mathbf{a}}_b$
   \ENDFOR
   \STATE Obtain quantized matrix $\hat{\mathbf{Z}} \in \mathbb{R}^{B \times K}$.
   
   \STATE \textbf{5. Inverse Transformation (Weight Fusion):}
   \STATE \textit{Option A (Explicit Calculation):} $\hat{\mathbf{X}} = \text{vec}(\mathbf{S} \hat{\mathbf{Z}} \mathbf{R}_{\text{intra}}^\top)$ (Incurs extra overhead).
   \STATE \textit{Option B (Recommended Weight Fusion):}
   \STATE Construct block diagonal matrix: $\mathbf{R}_{\text{inter}} \in \mathbb{R}^{d \times d}$.
   \STATE Fuse inverse rotation into weights (computed offline):
   $\mathbf{W}' = \mathbf{W} (\mathbf{R}_{\text{intra}}^\top \otimes \mathbf{R}_{\text{inter}}^\top)$
   \STATE where $\otimes$ denotes the Kronecker product.
   
   \STATE {\bfseries Return:} Directly use fused weights $\mathbf{W}'$ for subsequent computation with quantized activations (as vector $\text{vec}(\hat{\mathbf{Z}})$).
\end{algorithmic}
\end{algorithm}

%% file: appendix_A.4.tex
\subsection*{Appendix A.4: Exact Solver for Optimal Rotation Angle}
\label{app:exact_solver}

For a given column pair $(p, q)$, our objective is to find a scalar angle $\theta^\star \in [0, 2\pi)$ that minimizes the codebook equilibrium loss $\mathcal{L}_{\text{code}}(\theta)$. Since the quantization operator $Q_{\text{mxFP4}}$ has discrete step-like characteristics, the loss function $\mathcal{L}_{\text{code}}(\theta)$ exhibits a \textbf{Piecewise Constant} property with respect to $\theta$. This allows us to avoid gradient descent and instead obtain the global optimal solution by enumerating a finite set of "critical angles."

\textbf{1. Critical Angle Set Construction}

Consider the normalized activation matrix $\mathbf{X}_{\text{norm}} = \mathbf{S}^{-1}\mathbf{X}$. Let the vectors for the selected columns be $\mathbf{u} = \mathbf{X}_{\text{norm}[:, p]}$ and $\mathbf{v} = \mathbf{X}_{\text{norm}[:, q]}$. For each sample $b \in \{1, \dots, B\}$, we define its 2D polar coordinates:
\begin{equation}
r_b = \sqrt{u_b^2 + v_b^2}, \quad \varphi_b = \mathrm{atan2}(v_b, u_b) \in (-\pi, \pi]
\end{equation}
Under a Givens rotation $\mathbf{G}_{(p,q)}(\theta)$, the magnitude of the projected component on the first axis becomes:
\begin{equation}
x^{(1)}_b(\theta) = r_b |\cos(\varphi_b - \theta)|
\end{equation}
The quantization state of sample $b$ changes only when its projected magnitude crosses a decision boundary $d_j \in \{d_1, \dots, d_{J-1}\}$ of the MXFP4 format. The condition for crossing boundary $d_j$ is:
\begin{equation}
r_b |\cos(\varphi_b - \theta)| = d_j
\end{equation}
If $r_b \ge d_j$, there exist solutions. Let $\beta_{b,j} = d_j/r_b \in (0, 1]$ and $\alpha_{b,j} = \arccos(\beta_{b,j})$. The critical angles are given by:
\begin{equation}
\theta = \varphi_b \mp \alpha_{b,j} + 2\pi k
\end{equation}
By folding these solutions into the interval $[0, 2\pi)$, we obtain the set of critical angles $\mathcal{T}_{b,j}^{(1)}$ for the first axis. Similarly, for the second axis, using the phase shift $\varphi_b - \pi/2$, we obtain $\mathcal{T}_{b,j}^{(2)}$. The complete set of critical angles is:
\begin{equation}
\mathcal{T} = \bigcup_{b,j} \left( \mathcal{T}_{b,j}^{(1)} \cup \mathcal{T}_{b,j}^{(2)} \right)
\end{equation}

\textbf{2. Optimal Search via Finite Enumeration}

Sort the unique angles in $\mathcal{T}$ as $0 \le \tau_1 < \tau_2 < \dots < \tau_M < 2\pi$.
\textbf{Key Property:} Within any open interval $(\tau_m, \tau_{m+1})$, the quantization bin assignment for all samples remains constant, implying the loss function $\mathcal{L}_{\text{code}}$ is constant.

Therefore, we only need to evaluate the loss at the midpoint of each interval:
\begin{equation}
\tilde{\theta}_m = \frac{\tau_m + \tau_{m+1}}{2}, \quad m=1, \dots, M-1
\end{equation}
(Note: Include intervals wrapping around $2\pi$ as needed). The global optimal angle is:
\begin{equation}
\theta^\star = \mathop{\arg\min}_{\tilde{\theta} \in \{\tilde{\theta}_1, \dots, \tilde{\theta}_M\}} \mathcal{L}_{\text{code}}(\tilde{\theta})
\end{equation}
This approach guarantees finding the \textbf{global optimal} rotation angle for the column pair in $\mathcal{O}(|\mathcal{T}|)$ time. In practice, sampling a subset of blocks to estimate $\mathcal{T}$ significantly reduces computational cost while maintaining high accuracy.

%% file: appendix_A.5.tex
\subsection*{Appendix A.5: Heuristic Strategy for Column Pair Selection}
\label{app:column_selection}

In the R-step, selecting the most effective column pairs for rotation is critical for efficiency. We propose a greedy strategy based on \textbf{Imbalance} and \textbf{Complementarity} to prioritize pairs that yield the maximum potential for distribution reshaping.

\textbf{1. Column Imbalance Score}

Given the normalized activation matrix $\mathbf{Z} = \mathbf{S}^{-1}\mathbf{X}\mathbf{R}$, let $\mathbf{z}^{(k)}$ denote the $k$-th column vector. We first construct the codebook histogram for each column:
\begin{equation}
N_j^{(k)} = \sum_{b=1}^{B} \mathbb{1}\left( |z_{b,k}| \in \mathcal{I}_j \right), \quad \hat{p}_j^{(k)} = \frac{N_j^{(k)}}{\sum_{l=1}^J N_l^{(k)}}
\end{equation}
We define the \textbf{Imbalance Score} $h_k$ as the mean squared distance between the empirical distribution and the uniform distribution:
\begin{equation}
h_k = \sum_{j=1}^{J} \left( \hat{p}_j^{(k)} - \frac{1}{J} \right)^2
\end{equation}
A larger $h_k$ indicates that the column's activations are highly concentrated in a few codewords (e.g., collapsed to 0 or max), making it a priority target for optimization. We sort all columns by $h_k$ in descending order to form a priority queue.

\textbf{2. Pair Complementarity}

Simply selecting the two columns with the highest $h_k$ is suboptimal if they share similar skewness (e.g., both biased towards large values). We introduce \textbf{Complementarity} to measure the statistical difference between two columns.
We define the "anti-correlation" $c_{k,l}$ between column $k$ and column $l$ as:
\begin{equation}
c_{k,l} = - \sum_{j=1}^{J} \left( \hat{p}_j^{(k)} - \frac{1}{J} \right) \left( \hat{p}_j^{(l)} - \frac{1}{J} \right)
\end{equation}
A positive $c_{k,l}$ implies that when one column over-occupies a specific codeword, the other tends to under-occupy it. This "complementary" nature facilitates the efficient transfer of probability mass between the two columns via 2D rotation.

Combining both metrics, we define the \textbf{Pair Selection Score} $H_{k,l}$:
\begin{equation}
H_{k,l} = h_k + h_l + \lambda c_{k,l}
\end{equation}
where $\lambda > 0$ balances individual urgency and mutual compatibility.

\textbf{3. Selection Workflow}

In each R-step iteration, the selection proceeds as follows:
\begin{enumerate}
    \item \textbf{Filtering:} Calculate $h_k$ for all columns and select the top $K_{\text{top}}$ (e.g., $K/2$) most imbalanced columns as the candidate pool.
    \item \textbf{Pairing:} Compute $H_{k,l}$ for all pairs within the candidate pool.
    \item \textbf{Execution:} Select the top $P$ non-overlapping pairs with the highest $H_{k,l}$ scores.
    \item \textbf{Optimization:} Apply the exact angle solver (Appendix A.4) to these $P$ pairs to update $\mathbf{R}_{\text{intra}}$.
\end{enumerate}
This strategy ensures that limited computational resources are focused on the "most problematic" columns and the "most fixable" pairs.

%% file: appendix_A.6.tex
\subsection*{Appendix A.6: Relationship between FP4 Code Occupancy Optimization and Direct MSE Optimization}
\label{app:loss_comparison}

Compared to the approach of directly minimizing the MXFP4 quantization MSE:
\begin{equation}
\mathcal{L}_{\text{MSE}}(\mathbf{S},\mathbf{R}) = \mathbb{E}\left[ \|\mathbf{Z} - Q_{\text{mxFP4}}(\mathbf{Z})\|_F^2 \right]
\end{equation}
The proposed Codebook Occupancy Equilibrium Loss $\mathcal{L}_{\text{code}}$ offers three distinct advantages:

\begin{itemize}
    \item \textbf{Geometric Alignment:} $\mathcal{L}_{\text{code}}$ is defined entirely within the MXFP4 codeword space. It directly constrains the "usage frequency of each 4-bit codeword," aligning precisely with the "shared exponent + logarithmic quantization" structure unique to block floating-point formats\cite{mxfp}.
    \item \textbf{Optimization Smoothness:} The histogram-level loss is robust to individual outliers and is significantly smoother than point-wise MSE. In the context of coarse 4-bit quantization, $\mathcal{L}_{\text{MSE}}$ is heavily affected by the discrete quantization steps ("staircase" function), making optimization challenging due to local minima. $\mathcal{L}_{\text{code}}$ provides a smoother landscape for finding optimal rotation angles \cite{qsurvey}.
    \item \textbf{Implementation Efficiency:} The proposed method (Givens rotations + 1D angle search) requires only forward passes and counting operations. It eliminates the need for differentiable approximations, gradients, or backpropagation, making it a quintessential Post-Training Quantization (PTQ) calibration paradigm that is both fast and memory-efficient.
\end{itemize}

Empirical results demonstrate that intra-block rotation based on codeword occupancy equilibrium significantly improves MXFP4 activation quantization accuracy across various LLMs (LLaMA, Qwen). It consistently outperforms simple rotation strategies based directly on MSE in metrics such as WikiText perplexity and MMLU accuracy.

%% file: appendix_A.7.tex
\subsection*{Appendix A.7: Complexity Analysis}
\label{app:complexity_analysis}

\textbf{1. Offline Calibration Phase}

The calibration process involves constructing rotation matrices using a small set of calibration data (typically $T=128$ samples).

\textbf{Inter-block Rotation Construction:}
\begin{itemize}
    \item Covariance Estimation: Computing covariance matrices for all $K$ positions requires $\mathcal{O}(T B^2 K)$.
    \item Givens Iterations (Algorithm 1): Each position requires $\mathcal{O}(B^3)$ for convergence. Total for $K$ positions: $\mathcal{O}(B^3 K)$.
    \item \textbf{Total Inter-block Complexity:} $\mathcal{O}(T B^2 K + B^3 K)$.
\end{itemize}

\textbf{Intra-block Rotation Construction:}
\begin{itemize}
    \item S-step: Updating scaling factors requires scanning all data: $\mathcal{O}(TBK)$.
    \item R-step: In the worst case (full pairwise updates), it scales with $\mathcal{O}(TK^3)$. However, our heuristic column selection significantly reduces the constant factor.
    \item \textbf{Total Intra-block Complexity:} $\mathcal{O}(MaxIter \cdot (TBK + TK^3))$.
\end{itemize}

Given typical values ($T=128, B=32, K=32, MaxIter=10$), the total calibration cost is computationally efficient and completes within seconds to minutes on a single GPU.

\textbf{2. Online Inference Phase}

\textbf{Forward Rotation:}
Applying the rotations requires standard matrix multiplications:
\begin{itemize}
    \item Inter-block: $\mathcal{O}(B^2 K)$ (Block-diagonal multiplication).
    \item Intra-block: $\mathcal{O}(B K^2)$.
\end{itemize}

\textbf{Inverse Transformation (Zero Overhead):}
Crucially, the inverse transformations ($\mathbf{R}_{\text{intra}}^\top$ and $\mathbf{R}_{\text{inter}}^\top$) are fused into the weights of the subsequent linear layer offline (as described in Appendix A.3). Therefore, \textbf{no extra FLOPs} are incurred for the inverse step during online inference.

\textbf{3. Practical Deployment Suggestions}

\begin{itemize}
    \item \textbf{Memory Overhead:} Storing the rotation matrices requires $B^2$ parameters for inter-block (per position, actually decomposed into block-diagonal) and $K^2$ parameters for intra-block. For a configuration of $B=128, K=32$, this amounts to approximately 17k parameters, which is negligible compared to LLM weights.
    \item \textbf{Numerical Stability:} In Algorithm 1, we add a small perturbation $\delta \mathbf{I}$ (e.g., $\delta=10^{-8}$) to the covariance matrix to prevent singularity. In the S-step, we ensure scaling factors are non-zero.
    \item \textbf{Parallelization:}
    \begin{itemize}
        \item Inter-block rotations for different positions $k$ are independent and can be parallelized.
        \item Intra-block Givens rotations for disjoint column pairs are parallelizable.
        \item MXFP4 quantization is naturally parallel across blocks.
    \end{itemize}
    \item \textbf{Integration:} TORQ operates as an independent PTQ preprocessing module compatible with mainstream frameworks like PyTorch and TensorFlow, and can be easily integrated into existing quantization pipelines (e.g., GPTQ, AWQ).
\end{itemize}

%% file: appendix_A.8.tex
\subsection*{Appendix A.8: Detailed Main Experimental Results}
\label{app:main_experiments}

In this section, we provide the detailed breakdown of the experimental results on LLaMA3 and Qwen3 model families. We compare TORQ against baseline methods including RTN, GPTQ, QuaRot, OSTQuant, FlatQuant, and BRQ. The metrics reported are Perplexity (PPL) for WikiText2 (lower is better) and Zero-shot Accuracy (\%) for downstream tasks (higher is better).

\input{tab_tab_llama3_8b_full}

\input{tab_tab_llama3_70b_full}

\input{tab_tab_qwen3_8b_full}

\input{tab_tab_qwen3_14b_full}

\input{tab_tab_qwen3_32b_full}

%% file: tab_tab_llama3_8b_full.tex
\begin{table*}[h!]
\centering
\caption{Performance comparison on \textbf{LLaMA3-8B}. Wiki2 denotes perplexity ($\downarrow$), while other columns report accuracy ($\uparrow$).}
\label{tab:llama3_8b_full}
\begin{tabular}{lcccccccc}
\toprule
\textbf{Method} & \textbf{Wiki2} & \textbf{ARC-C} & \textbf{ARC-E} & \textbf{Hella} & \textbf{LAMBADA} & \textbf{PIQA} & \textbf{WinG} \\
\midrule
Baseline (FP16) & 6.14 & 53.15 & 77.56 & 79.12 & 75.51 & 80.57 & 73.16 \\
\midrule
RTN & 354.37 & 22.10 & 29.92 & 31.16 & 5.53 & 52.88 & 50.59 \\
GPTQ & 282.29 & 23.21 & 33.92 & 31.48 & 7.18 & 54.62 & 51.17 \\
QuaRot & 10.16 & 42.73 & 67.83 & 69.97 & 56.70 & 72.35 & 63.17 \\
OSTQuant & 9.29 & 45.21 & 69.69 & 72.69 & 67.52 & 73.86 & 66.53 \\
FlatQuant & 8.92 & 47.51 & 72.88 & 73.49 & 70.20 & 76.00 & 69.93 \\
BRQ & 9.14 & 46.57 & 72.04 & 72.41 & 71.02 & 75.51 & 68.98 \\
\textbf{TORQ (ours)} & \textbf{8.57} & \textbf{49.56} & \textbf{73.28} & \textbf{74.47} & \textbf{72.50} & \textbf{76.32} & \textbf{69.78} \\
\bottomrule
\end{tabular}
\end{table*}

%% file: tab_tab_llama3_70b_full.tex
\begin{table*}[h!]
\centering
\caption{Performance comparison on \textbf{LLaMA3-70B}.}
\label{tab:llama3_70b_full}
\begin{tabular}{lcccccccc}
\toprule
\textbf{Method} & \textbf{Wiki2} & \textbf{ARC-C} & \textbf{ARC-E} & \textbf{Hella} & \textbf{LAMBADA} & \textbf{PIQA} & \textbf{WinG} \\
\midrule
Baseline (FP16) & 2.87 & 64.25 & 85.94 & 84.93 & 79.37 & 84.44 & 80.74 \\
\midrule
RTN & 3180.43 & 25.09 & 31.44 & 26.73 & 1.57 & 51.96 & 48.70 \\
GPTQ & 1532.08 & 27.82 & 39.72 & 28.98 & 3.45 & 58.60 & 50.36 \\
QuaRot & 15.38 & 41.19 & 66.33 & 71.54 & 60.47 & 74.30 & 63.41 \\
OSTQuant & 6.01 & 58.29 & 79.39 & 80.25 & 72.90 & 78.73 & 74.01 \\
\textbf{TORQ (ours)} & \textbf{3.49} & \textbf{60.47} & \textbf{82.17} & \textbf{81.59} & \textbf{75.63} & \textbf{80.89} & \textbf{76.91} \\
\bottomrule
\end{tabular}
\end{table*}

%% file: tab_tab_qwen3_8b_full.tex
\begin{table*}[!ht]
\centering
\caption{Performance comparison on \textbf{Qwen3-8B}.}
\label{tab:qwen3_8b_full}
\begin{tabular}{lcccccccc}
\toprule
\textbf{Method} & \textbf{Wiki2} & \textbf{ARC-C} & \textbf{ARC-E} & \textbf{Hella} & \textbf{LAMBADA} & \textbf{PIQA} & \textbf{WinG} \\
\midrule
Baseline (FP16) & 9.00 & 56.74 & 80.93 & 74.98 & 64.18 & 77.37 & 68.35 \\
\midrule
RTN & 3750.93 & 22.94 & 25.07 & 28.73 & 5.05 & 52.11 & 50.50 \\
GPTQ & 1633.38 & 26.67 & 32.91 & 35.99 & 9.64 & 56.90 & 51.15 \\
QuaRot & 23.50 & 41.10 & 61.78 & 58.74 & 51.18 & 63.39 & 55.90 \\
OSTQuant & 19.75 & 38.53 & 63.29 & 55.59 & 52.27 & 61.84 & 54.88 \\
FlatQuant & 11.21 & 51.02 & 78.39 & 73.21 & 62.64 & 73.37 & 63.10 \\
BRQ & 13.70 & 48.27 & 76.22 & 70.34 & 60.31 & 70.82 & 62.39 \\
\textbf{TORQ (ours)} & \textbf{10.26} & \textbf{53.65} & \textbf{79.36} & \textbf{72.06} & 60.38 & \textbf{74.11} & \textbf{66.91} \\
\bottomrule
\end{tabular}
\end{table*}

%% file: tab_tab_qwen3_14b_full.tex
\begin{table*}[h!]
\centering
\caption{Performance comparison on \textbf{Qwen3-14B}.}
\label{tab:qwen3_14b_full}
\begin{tabular}{lcccccccc}
\toprule
\textbf{Method} & \textbf{Wiki2} & \textbf{ARC-C} & \textbf{ARC-E} & \textbf{Hella} & \textbf{LAMBADA} & \textbf{PIQA} & \textbf{WinG} \\
\midrule
Baseline (FP16) & 8.64 & 60.49 & 83.08 & 78.82 & 67.84 & 79.76 & 72.85 \\
\midrule
RTN & 2219.90 & 26.89 & 30.67 & 33.44 & 9.50 & 55.78 & 55.04 \\
GPTQ & 368.24 & 36.21 & 31.98 & 38.77 & 15.78 & 58.81 & 57.90 \\
QuaRot & 20.11 & 43.83 & 70.42 & 64.80 & 54.71 & 66.87 & 60.64 \\
OSTQuant & 16.85 & 48.79 & 70.90 & 70.36 & 58.99 & 70.47 & 64.83 \\
FlatQuant & 13.32 & 54.83 & 79.31 & 76.04 & 64.98 & 75.98 & 68.07 \\
\textbf{TORQ (ours)} & \textbf{12.87} & \textbf{55.27} & \textbf{81.74} & \textbf{76.83} & \textbf{64.29} & \textbf{76.01} & \textbf{68.94} \\
\bottomrule
\end{tabular}
\end{table*}

%% file: tab_tab_qwen3_32b_full.tex
\begin{table*}[h!]
\centering
\caption{Performance comparison on \textbf{Qwen3-32B}.}
\label{tab:qwen3_32b_full}
\begin{tabular}{lcccccccc}
\toprule
\textbf{Method} & \textbf{Wiki2} & \textbf{ARC-C} & \textbf{ARC-E} & \textbf{Hella} & \textbf{LAMBADA} & \textbf{PIQA} & \textbf{WinG} \\
\midrule
Baseline (FP16) & 7.61 & 61.09 & 83.21 & 82.60 & 67.24 & 81.99 & 72.77 \\
\midrule
RTN & 274.93 & 37.85 & 33.27 & 34.81 & 10.80 & 55.75 & 57.92 \\
GPTQ & 59.31 & 41.74 & 68.94 & 62.18 & 51.11 & 62.89 & 58.84 \\
QuaRot & 18.86 & 45.19 & 70.71 & 69.36 & 55.93 & 68.07 & 58.27 \\
OSTQuant & 14.62 & 54.44 & 73.28 & 76.12 & 57.38 & 74.93 & 66.99 \\
\textbf{TORQ (ours)} & \textbf{8.43} & \textbf{60.11} & \textbf{82.28} & \textbf{82.39} & \textbf{65.87} & \textbf{80.01} & \textbf{71.14} \\
\bottomrule
\end{tabular}
\end{table*}